\documentclass{article}
\usepackage{spconf,amsmath,graphicx,amssymb,enumitem,float}
\usepackage{url}
\usepackage{xpatch}
\usepackage{siunitx} % add in preamble
\usepackage{calc,amsfonts,amssymb,amsmath,bm,color,graphicx,cite,epstopdf,nicefrac,stmaryrd}
\usepackage{psfrag,subfigure,float,epstopdf,bbold}
\usepackage[algoruled,linesnumbered]{algorithm2e}
\usepackage{multicol}
\usepackage{multirow,comment}
\usepackage{siunitx}
\usepackage{algorithmic}
\usepackage[normalem]{ulem}
\usepackage{mdframed}
\usepackage{footnote}
\makesavenoteenv{tabular}
\usepackage{booktabs}
\usepackage{subcaption} % add this in your preamble
\usepackage{float}
\usepackage{amsthm}  % <- add this
\usepackage{wrapfig}

\newcommand\norm[1]{\lVert#1\rVert}

\usepackage{xcolor}
\definecolor{orange}{RGB}{255,107,0}
\definecolor{green}{RGB}{50,170,50}
\definecolor{purple}{RGB}{255,0,255}

\usepackage{tikz}
\usetikzlibrary{bayesnet}
\usetikzlibrary{arrows}
\usepackage{color}
\usepackage{graphicx}
\usepackage{caption}
\usepackage{wrapfig}
\usepackage{amsthm}
\usetikzlibrary{backgrounds}
\usepackage{optidef}

\newtheorem{theorem}{Theorem}

\newtheorem{corollary}{Corollary}

\usepackage{xspace}
\usepackage{bbm}
\usepackage{dsfont}
\usepackage{bm}

\newcommand{\Cc}{\mathcal{C}}

\newcommand{\Ec}{\mathcal{E}}

\newcommand{\Lc}{\mathcal{L}}

\newcommand{\ev}{{\bf e}}

\newcommand{\xv}{{\bm x}}
\newcommand{\xvt}{\tilde{\bm x}}
\newcommand{\xvh}{\hat{\bm x}}
\newcommand{\wv}{{\bf w}}

\newcommand{\dv}{{\bm d}}
\newcommand{\yv}{{\bm y}}
\newcommand{\zv}{{\bm z}}

\newcommand{\nv}{{\bf n}}

\newcommand{\cv}{{\bf c}}

\DeclareMathOperator\R{\mathbb{R}}

\def\textiid{i.i.d.\@\xspace}
\newcommand\iid{\ifmmode\text{ i.i.d. } \else \textiid \fi}

\DeclarePairedDelimiter\parens{\lparen}{\rparen}  
\DeclarePairedDelimiter\abs{\lvert}{\rvert}

\DeclarePairedDelimiter\braces{\lbrace}{\rbrace}

\DeclarePairedDelimiter\angles{\langle}{\rangle}

\renewcommand{\P}[1]{\mathbb{P}\parens*{#1}}

\newcommand\eat[1]{}

\usepackage{enumitem}

\usepackage[]{hyperref} % keep this as the LAST package

\title{Perception-based Image Denoising via Generative Compression}
\name{
Nam Nguyen \qquad
Thinh Nguyen \qquad
Bella Bose
\thanks{This work was supported by the National Science Foundation under Grant No. CCF:SHF:2417898.}
}
\address{
School of Electrical and Computer Engineering, Oregon State University, Corvallis, OR 97331, USA
}

\begin{document}
%\ninept
%
\maketitle

\begin{abstract}
Image denoising aims to remove noise while preserving structural details and perceptual realism, yet distortion-driven methods often produce over-smoothed reconstructions, especially under strong noise and distribution shift. 
This paper proposes a generative compression framework for perception-based denoising, where restoration is achieved by reconstructing from entropy-coded latent representations that enforce low-complexity structure, while generative decoders recover realistic textures via perceptual measures such as learned perceptual image patch similarity (LPIPS) loss and Wasserstein distance. 
Two complementary instantiations are introduced: (i) a conditional Wasserstein GAN (WGAN)-based compression denoiser that explicitly controls the rate-distortion-perception (RDP) trade-off, and (ii) a conditional diffusion-based reconstruction strategy that performs iterative denoising guided by compressed latents. 
We further establish non-asymptotic guarantees for the compression-based maximum-likelihood denoiser under additive Gaussian noise, including bounds on reconstruction error and decoding error probability. 
Experiments on synthetic and real-noise benchmarks demonstrate consistent perceptual improvements while maintaining competitive distortion performance.
\end{abstract}


%===========================================================================================
\section{Introduction}
\label{sec:intro}
Image denoising is a key problem in image processing with applications spanning low-light photography, microscopy, and scientific imaging. The objective is to remove noise while preserving structural details and perceptual realism.
Classical methods rely on hand-crafted priors such as sparsity and nonlocal self-similarity, with BM3D being a prominent example~\cite{dabov2007bm3d}.
More recently, deep discriminative models, including DnCNN~\cite{zhang2017beyond} and FFDNet~\cite{zhang2018ffdnet}, have achieved substantial improvements by learning direct mappings from noisy to clean images.
However, approaches optimized primarily for pixel-wise distortion metrics such as mean squared error (MSE) or peak signal-to-noise ratio (PSNR) often produce over-smoothed reconstructions and loss of fine textures, particularly under strong noise or distribution shift \cite{blau2019rethinking}.

\begin{figure}[t]
    \centering
    \includegraphics[width=0.45\textwidth]{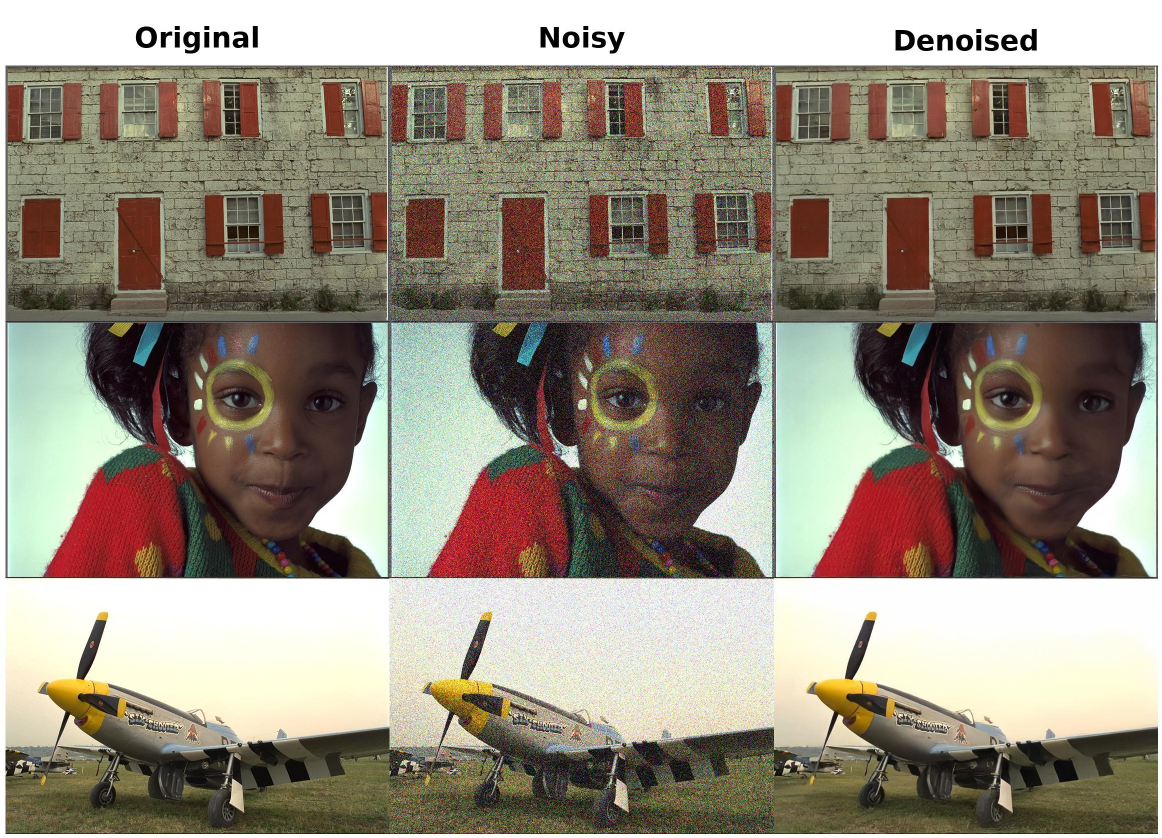}
    \caption{Samples generated by our methods on KODAK images.
    From left to right: original, noisy observation with Gaussian noise $\mathcal{N}(0,\sigma^2)$ with $\sigma=50$, and denoised output.}
    \label{fig:show_method_intro}
\end{figure}

Perceptual image restoration seeks to overcome these limitations by explicitly prioritizing visual quality.
The rate-distortion-perception (RDP) tradeoff established in~\cite{blau2019rethinking} formalizes that minimizing distortion alone is insufficient for perceptual realism, motivating the use of perceptual losses and generative models.
Conditional Wasserstein GAN (WGAN)~\cite{arjovsky2017wasserstein,mentzer2020high} has been widely adopted for perceptual restoration; however, adversarial training is often unstable and may inadequately capture the full diversity of natural image statistics.
More recently, diffusion models have emerged as powerful generative priors, offering stable training and strong mode coverage by learning to reverse a gradual noise-corruption process~\cite{ho2020denoising}.

Denoising and compression are closely connected through the principle that natural signals exhibit low complexity, whereas noise is largely unstructured.
Compression-based denoising exploits this idea by interpreting denoising as a projection onto a low-complexity signal class induced by a lossy codec.
From an information-theoretic perspective, operating a compressor at a distortion level matched to the noise strength suppresses high-complexity noise while preserving structured content~\cite{weissman2005empirical}.
Recent works have instantiated this principle using neural compression models~\cite{zafari2025decompress,zafari2025zeroshot,nguyen2026cross}, but limited perceptual modeling constrains their performance on natural images.

In this paper, we propose a generative compression framework for perception-based image denoising based on entropy-coded latent representations.
We introduce two complementary approaches.
First, a conditional WGAN-based compression denoiser encodes noisy images into compressed latents and reconstructs denoised outputs using conditional adversarial training and perceptual distortion measured by learned perceptual image patch similarity (LPIPS)~\cite{ding2020image}, enabling explicit control over rate, distortion, and perceptual quality.
Second, we propose a conditional diffusion-based reconstruction strategy that performs iterative denoising conditioned on compressed latents, combining the stability and expressiveness of diffusion priors with compression-induced global consistency. On the theory side, we formulate a perception-constrained compression framework for image denoising that explicitly connects denoising performance to RDP theory, and establish non-asymptotic upper bounds on the reconstruction error and the decoding error probability of a compression-based maximum- likelihood denoiser under a Gaussian noise model. Experimental results on synthetic and real-noise benchmarks demonstrate that the proposed methods achieve improved perceptual fidelity while maintaining competitive distortion performance.

%===========================================================================================
\section{Denoising via Perception-Constrained Lossy Compression}

\subsection{Problem Description}
Denoising is a fundamental problem in signal processing and has recently attracted renewed interest in machine learning.
Let $\xv=(x_1,\ldots,x_n)\in\mathbb{R}_+^n$ denote an unknown non-negative signal, and let $\yv=(y_1,\ldots,y_n)$ be its noisy observation.
We assume a memoryless and homogeneous noise model, where observations are conditionally independent given $\xv$, i.e.,
$\yv \sim \prod_{i=1}^n p(y_i | x_i)$.
The objective of denoising is to estimate the underlying signal $\xv$ from $\yv$.

Let $\mathcal{Q}\subset\mathbb{R}^n$ denote a signal class of interest, such as vectorized natural images.
A lossy compression scheme for $\mathcal{Q}$ is defined by an encoder-decoder pair $(f,g)$, where
$f:\mathcal{Q}\rightarrow\{1,\ldots,2^R\}$ and $g:\{1,\ldots,2^R\}\rightarrow\mathbb{R}^n$.
Its performance is characterized by the compression rate $R$, the reconstruction distortion $D$, and the perceptual quality $P$, which measures the discrepancy between the distributions of reconstructed and original signals.

Following RDP theory~\cite{blau2019rethinking}, the trade-off between distortion and perceptual quality at a fixed rate can be characterized by the distortion-perception (DP) function
\begin{align}
\label{prob:D_P_func}
    D(P)
    = \inf_{f,g} \; \mathbb{E}\!\left[\|\xv-\tilde{\xv}\|^2\right]
    \quad \text{s.t.} \quad
    W_2(p_{\xv},p_{\tilde{\xv}})\le P,
\end{align}
where $\tilde{\xv}=g(f(\xv))$, $\|\cdot\|^2$ denotes the MSE distortion, and $W_2(\cdot,\cdot)$ is the Wasserstein-2 distance between the distributions of the original and reconstructed signals.

The unconstrained MSE-optimal estimator is the minimum MSE (MMSE) estimator $\xv^*=\mathbb{E}[\xv|\tilde{\xv}]$, which attains distortion
$D^*=\mathbb{E}[\|\xv-\xv^*\|^2]$ and perception index
$P^*=W_2(p_{\xv},p_{\xv^*})$.
As shown in~\cite{freirich2021theory}, $D(P)$ decreases monotonically with $P$ until $P=P^*$, beyond which it remains constant at $D^*$.

\begin{theorem}[\cite{freirich2021theory}]
The DP function in~\eqref{prob:D_P_func} admits that
\begin{equation*}
    D(P)=D^*+\bigl[(P^*-P)_+\bigr]^2,
\end{equation*}
where $(x)_+=\max(0,x)$.
\end{theorem}

\subsection{Perceptual Compression-based Denoising}
Given a lossy compression code $(f,g)$ with rate $R$, the decoder induces a codebook
$\mathcal{C}=\{\cv_m=g(m):m=1,\ldots,2^R\}\subset\mathbb{R}^n$.
For any $\xv\in\mathcal{Q}$, the reconstruction $\tilde{\xv}=g(f(\xv))$ satisfies the distortion-perception trade-off
$\|\xv-\tilde{\xv}\|^2=D(P)$ with $d_p(p_{\xv},p_{\tilde{\xv}})\le P$.

Building on~\cite{zafari2025zeroshot}, compression-based denoising can be interpreted as a structured maximum-likelihood (ML) estimator.
Given a noisy observation $\yv\sim\prod_{i=1}^n p(y_i|x_i)$, the denoised estimate is obtained by selecting the most likely codeword, $\xvh=\arg\min_{\cv_m\in\mathcal{C}}\;\Lc(\cv_m;\yv)$ and $\Lc(\cv_m;\yv)=-\sum_{i=1}^n\log p(y_i|c_{mi})$, 
% \begin{align*}
% \xvh=\arg\min_{\cv_m\in\mathcal{C}}\;\Lc(\cv_m;\yv); \hspace{0.1cm}
% \Lc(\cv_m;\yv)=-\sum_{i=1}^n\log p(y_i|c_{mi}),
% \end{align*}
where $\mathcal{C}$ is designed to represent clean signals.
This formulation exploits the fact that clean images admit more compact representations than noisy observations, yielding a structure-aware denoiser applicable across noise models. For additive white Gaussian noise (AWGN), $\yv=\xv+\nv$ with $\nv\sim\mathcal{N}(\mathbf{0},\sigma^2\mathbf{I}_n)$, the estimator reduces to
\begin{equation}
\xvh=\arg\min_{\cv_m\in\mathcal{C}}\;\|\yv-\cv_m\|^2,
\label{eq:comp-denoiser-awgn}
\end{equation}
i.e., projection of the noisy observation onto the nearest codeword in the compression codebook.

\subsection{Theoretical Result}
We analyze the performance of compression-based ML denoising under additive white Gaussian noise (AWGN).
The following theorem provides a non-asymptotic bound on the reconstruction error in terms of the compression rate, distortion, and perceptual quality.

\begin{theorem}
\label{thm:MSE_Gaussian_noise}
Let $\xv\in\mathcal{Q}$ and consider a lossy compression scheme $(f,g)$ operating at rate $R$, distortion $D$, and perception level $P$.
Given noisy observations $\yv=\xv+\nv$, where $\nv\sim\mathcal{N}(\mathbf{0},\sigma^2\mathbf{I}_n)$, let $\xvh$ denote the output of the compression-based ML denoiser in~\eqref{eq:comp-denoiser-awgn}.
Then, with probability at least $1-2^{-\eta R+2}$ for any $\eta\in(0,1)$,
\begin{align*}
\sqrt{D(P)} \leq \|\xv-\xvh\| \le \sqrt{D(P)}  + 2\sigma \sqrt{(2 \ln2) R}(1+2\sqrt{\eta}),
\end{align*}
Moreover, the error probability of the ML denoiser satisfies
\begin{align*}
P_e \le \frac{1}{2^R}\sum_{m=1}^{2^R}\sum_{\substack{v=1\\v\neq m}}^{2^R}
Q\!\left(
\frac{\frac{\|\cv_m-\cv_v\|}{2}
+\angles{\dv,\frac{\cv_v-\cv_m}{\|\cv_v-\cv_m\|}}}{\sigma}
\right),
\end{align*}
where $\dv=\tilde{\xv}-\xv$ and $Q(\cdot)$ denotes the Gaussian tail function.
\end{theorem}

\begin{proof}
The proof is provided in Appendix~\ref{app:proof_MSE_Gaussian_noise}.
\end{proof}

\begin{corollary}
Since $\|\dv\|=\sqrt{D(P)}$, a worst-case bound is obtained by
$\angles{\dv,\frac{\cv_v-\cv_m}{\|\cv_v-\cv_m\|}}\ge -\|\dv\|$, yielding
\begin{align*}
P_e \le \frac{1}{2^R}\sum_{m=1}^{2^R}\sum_{\substack{v=1\\v\neq m}}^{2^R} \!\!
Q\!\left(
\frac{\frac{\|\cv_m-\cv_v\|}{2}
-\sqrt{D(P)}}{\sigma}
\right) \!\!.
\end{align*}
\end{corollary}

%===========================================================================================
\section{Generative Compression-based Image Denoising}
  
In this section, we propose learning-based denoisers that prioritize perceptual quality by compressing noisy observations with neural generative compression models, where denoising is performed by reconstructing from a compact latent representation.

\subsection{Perceptual Denoising via Conditional WGAN-based Compression}
\label{subsec:cgan_compress}

\begin{figure}[t]
    \centering
    \includegraphics[width=0.46\textwidth]{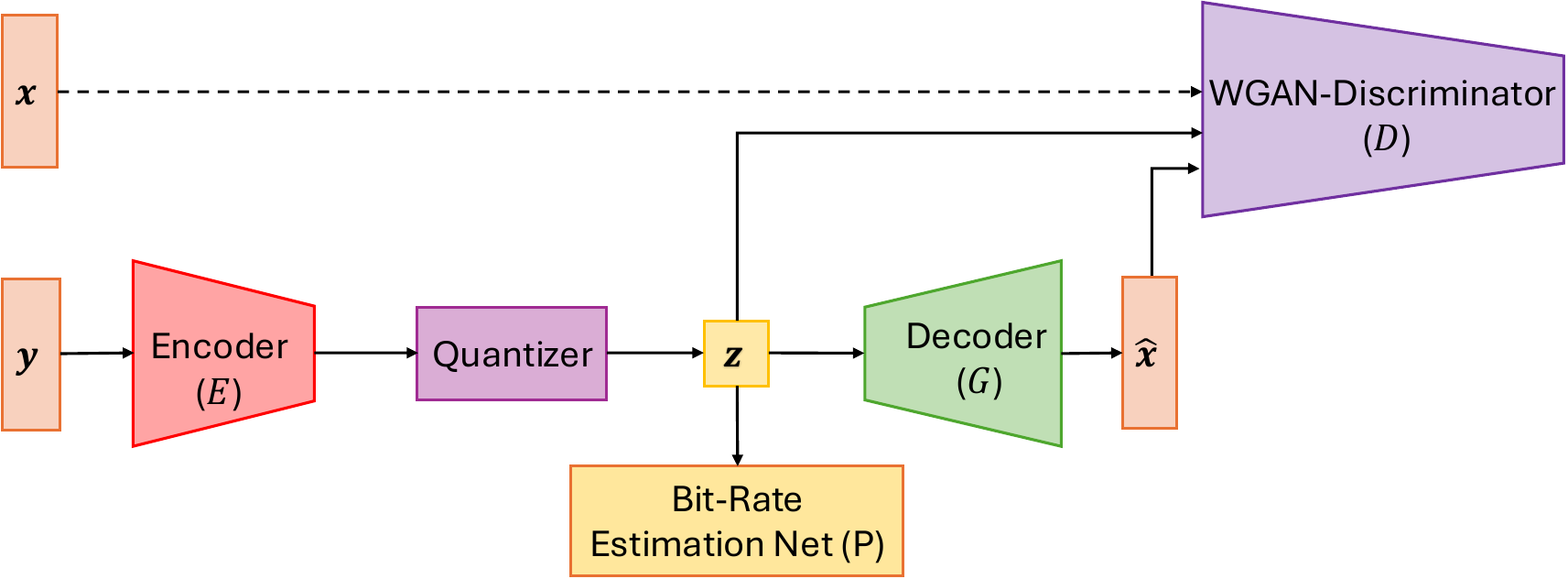}
    \caption{Conditional WGAN-based denoising framework.}
    \label{fig:architecture_cgan}
    \vspace{-0.5cm}
\end{figure}

Our approach builds on high-fidelity generative image compression~\cite{mentzer2020high} and integrates learned lossy compression with conditional adversarial training~\cite{mirza2014conditional}, as illustrated in Fig.~\ref{fig:architecture_cgan}.
An encoder-decoder pair $(E,G)$ maps a noisy input $\yv$ to a quantized latent representation $\zv=E(\yv)$ and reconstructs a denoised image $\hat{\xv}=G(\zv)$.
The latent $\zv$ is entropy-coded using a learned probability model $P$, yielding a rate $r(\zv)=-\log P(\zv)$.

To promote perceptual realism, we use a conditional Wasserstein GAN-based discriminator $D$ that distinguishes reconstructed images from clean images conditioned on $\zv$, following the non-saturating conditional WGAN objective \cite{arjovsky2017wasserstein, mirza2014conditional}.
The reconstruction distortion is defined as
\begin{equation*}
d(\xv,\hat{\xv})
=
k_M\,\mathrm{MSE}(\xv,\hat{\xv})
+
k_P\,\mathrm{LPIPS}(\xv,\hat{\xv}),
\end{equation*}
where LPIPS captures perceptual discrepancies not well reflected by pixel-wise losses~\cite{ding2020image}. The encoder, decoder, and entropy model are jointly optimized by minimizing
\begin{equation*}
\mathcal{L}_{EGP}
=
\mathbb{E}_{\yv}
\!\left[
\lambda\,r(\zv)
+
d(\xv,\hat{\xv})
-
\beta\log D(\hat{\xv},\zv)
\right],
\end{equation*}
where $\lambda$ controls the rate-distortion trade-off and $\beta$ regulates perceptual quality.
The discriminator is trained using
\begin{equation*}
\mathcal{L}_D
=
\mathbb{E}_{\xv}\!\left[-\log D(\xv,\zv)\right]
+
\mathbb{E}_{\yv}\!\left[-\log\!\left(1-D(\hat{\xv},\zv)\right)\right].
\end{equation*}

This formulation enables explicit control over the RDP trade-off for perceptual image denoising~\cite{blau2019rethinking}.

\subsection{Perceptual Denoising via Conditional Diffusion-based Compression}

\begin{figure}[t]
    \centering
    \includegraphics[width=0.45\textwidth]{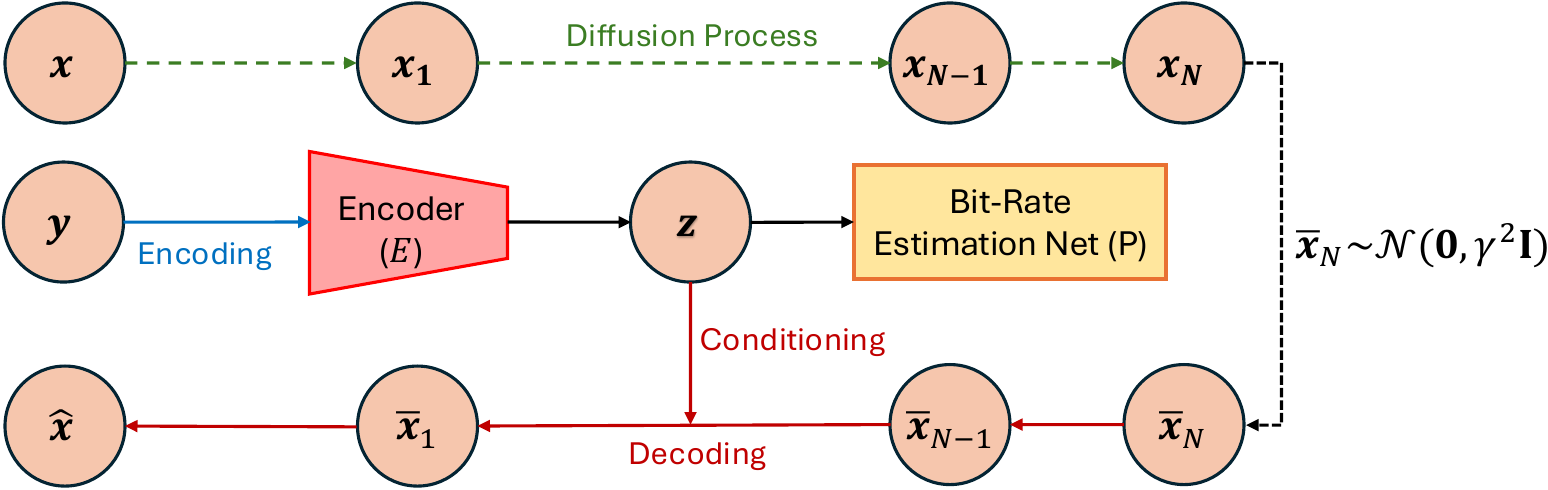}
    \caption{Conditional diffusion-based denoising architecture.}
    \label{fig:architecture_diffusion}
\end{figure}

\begin{figure*}[t]
    \centering
    \includegraphics[width=0.97\textwidth]{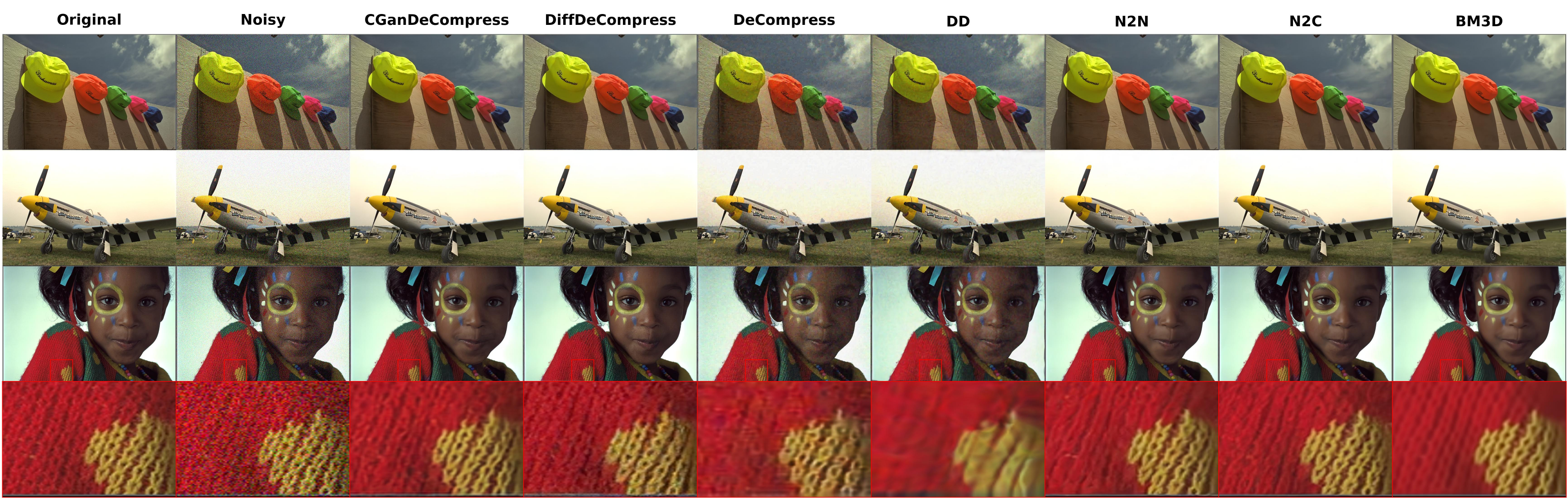}
    \caption{\small Visual comparison of different denoising methods on randomly selected KODAK images with noise level $\sigma=25$.
    The first two columns show the original images and their noisy observations, followed by results generated by our methods and popular baselines.
    The last two rows provide zoomed-in views for detailed comparison.
    Best viewed on screen.}
    \label{fig:methods_comparison_kodak_25}
\end{figure*}

\subsubsection{Denoising Diffusion Models}
Denoising diffusion models generate data by reversing a gradual noise corruption process~\cite{ho2020denoising}.
Starting from a clean image $\xv_0 = \xv$, the forward process produces noisy samples $\{\xv_n\}_{n=1}^N$ according to
\begin{equation}
q(\xv_n|\xv_{n-1})
= \mathcal{N}\!\left(\xv_n | \sqrt{1-\beta_n}\,\xv_{n-1}, \beta_n \mathbf{I}\right),
\end{equation}
where $\{\beta_n\}$ is a variance schedule.
This admits the closed form $\xv_n=\sqrt{\alpha_n}\xv_0+\sqrt{1-\alpha_n}\epsilon$, with $\epsilon\sim\mathcal{N}(\mathbf{0},\mathbf{I})$ and $\alpha_n=\prod_{i=1}^n(1-\beta_i)$.
As $n$ increases, $\xv_n$ converges to Gaussian noise. The reverse process is modeled as
\begin{equation}
p_\theta(\xv_{n-1}|\xv_n)
= \mathcal{N}\!\left(\xv_{n-1} | M_\theta(\xv_n,n), \beta_n \mathbf{I}\right),
\end{equation}
and is trained by predicting the injected noise using the Denoising Diffusion Probabilistic Model (DDPM) objective~\cite{ho2020denoising}
\begin{equation}
\mathcal{L}_{\mathrm{DDPM}}(\theta)
=
\mathbb{E}_{\xv_0,n,\epsilon}
\left\|
\epsilon-\epsilon_\theta(\xv_n,n)
\right\|^2.
\end{equation}
In inference, denoised samples are obtained by iteratively applying the learned reverse transitions, optionally using deterministic sampling.

\begin{figure*}[t]
    \centering
    \includegraphics[width=0.97\textwidth]{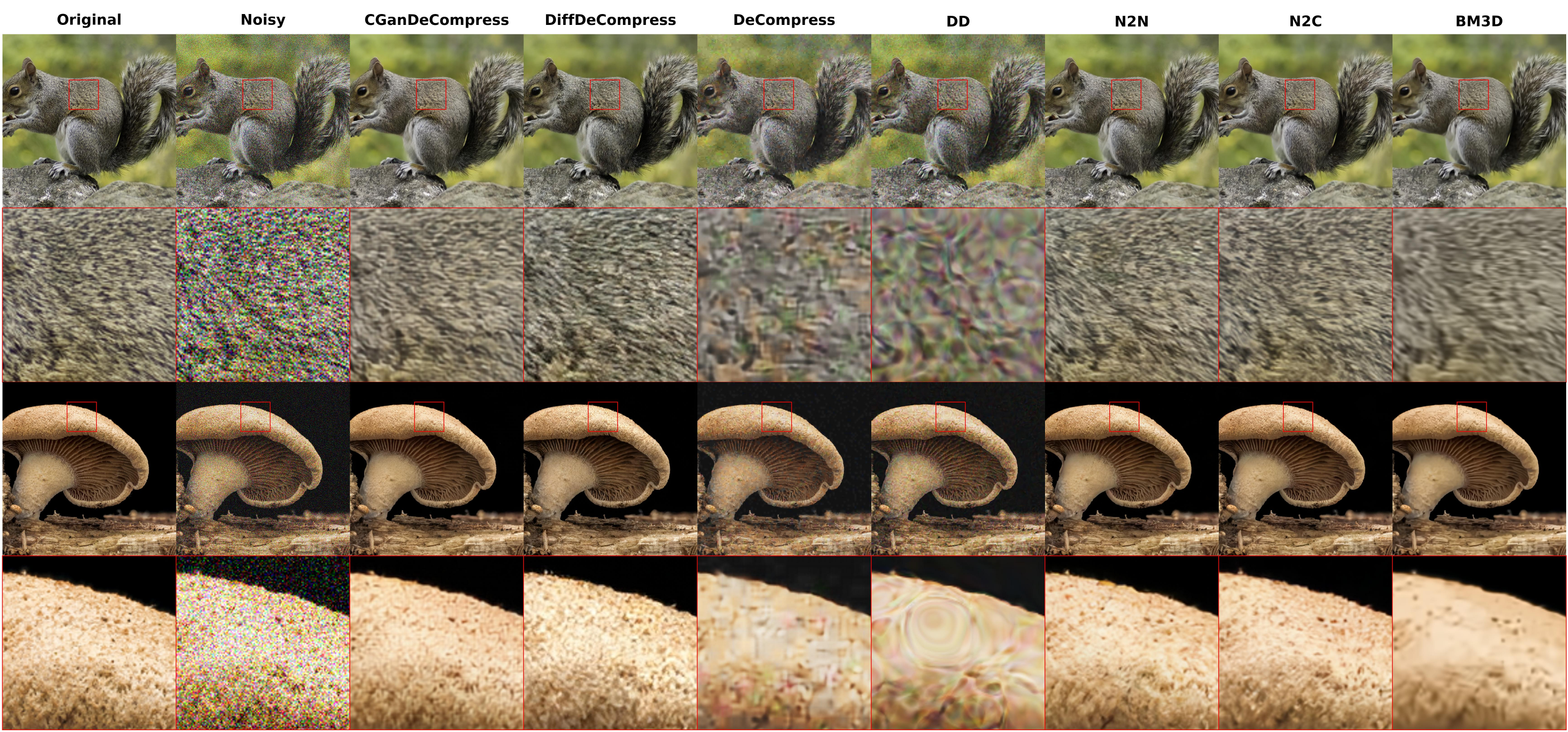}
    \caption{\small Visual comparison of denoising methods on randomly selected DIV2K images at noise level $\sigma=50$.
    The first two columns present the original images and their noisy observations, followed by the denoised results generated by our methods and popular baselines.
    Zoomed-in regions are shown in the last two rows for detailed inspection.
    Best viewed on screen.}
    \label{fig:methods_comparison_div2k_50}
\end{figure*}

\subsubsection{Conditional Diffusion for Perceptual Denoising}
To enable perceptual denoising, we adopt a conditional diffusion framework inspired by diffusion-based compression~\cite{yangmandt2024cdc}, as illustrated in Fig.~\ref{fig:architecture_diffusion}.
A noisy observation $\yv$ is encoded into a compact latent representation $\zv$ via a stochastic encoder $E(\zv|\yv)$.
Conditioned on $\zv$, the decoder models $p(\xv_0|\zv)$ through a diffusion process,
\begin{equation*}
p_\theta(\xv_{0:N}|\zv)
=
p(\xv_N)
\prod_{n=1}^{N}
\mathcal{N}\!\left(\xv_{n-1} | M_\theta(\xv_n,\zv,n), \beta_n \mathbf{I}\right),
\end{equation*}
where conditioning guides the denoising trajectory toward semantically consistent reconstructions, while intermediate latent variables enable the synthesis of fine-scale details.

Training minimizes a rate-distortion objective
\[
\mathbb{E}_{\zv\sim E(\zv|\xv_0)}
\big[
-\log p(\xv_0|\zv) - \lambda \log p(\zv)
\big],
\]
where $\lambda$ controls the rate constraint.
Since $p(\xv_0|\zv)$ is intractable, we optimize a variational upper bound on the diffusion negative log-likelihood~\cite{ho2020denoising},
\begin{equation*}
L_{\mathrm{upper}}(\xv_0|\zv)
=
\mathbb{E}_{\xv_0,n,\epsilon}
\frac{\alpha_n}{1-\alpha_n}
\left\|
\xv_0-\mathcal{X}_\theta(\xv_n,\zv,n/N)
\right\|^2.
\end{equation*}

To further enhance perceptual quality, we also incorporate an explicit LPIPS loss~\cite{ding2020image}.
Let $\bar{\xv}_0=\mathcal{X}_\theta(\xv_n,\zv,n/N)$ denote the decoded image, and define the total objective as
$L=\rho L_p+(1-\rho)L_c$, where
\[
L_p=\mathbb{E}[\mathrm{LPIPS}(\bar{\xv}_0,\xv_0)]; \hspace{0.01cm}
L_c=L_{\mathrm{upper}}(\xv_0|\zv)-\tfrac{\lambda}{1-\rho}\log P(\zv).
\]
This yields a perception-aware diffusion denoiser suitable for real-world image restoration.

\begin{table*}[t]
\footnotesize
\centering
\setlength{\tabcolsep}{3.0pt}        % tighter columns
\renewcommand{\arraystretch}{1.05}   % tighter rows

\begin{tabular}{c|l|ccccc|ccccc}
\hline
\multirow{2}{*}{\textbf{Noise}}
& \multirow{2}{*}{\textbf{Method}}
& \multicolumn{5}{c|}{\textbf{KODAK}}
& \multicolumn{5}{c}{\textbf{DIV2K}} \\
\cline{3-12}
&
& \textbf{FID}$\downarrow$
& \textbf{LPIPS}$\downarrow$
& \textbf{PI}$\downarrow$
& \textbf{PSNR}$\uparrow$
& \textbf{SSIM}$\uparrow$
& \textbf{FID}$\downarrow$
& \textbf{LPIPS}$\downarrow$
& \textbf{PI}$\downarrow$
& \textbf{PSNR}$\uparrow$
& \textbf{SSIM}$\uparrow$ \\
\hline\hline

% ========================= Sigma = 15 =========================
\multirow{8}{*}{$\sigma=15$}

& Noisy
& 90.2527 & 0.3225 & 4.4103 & 24.7506 & 0.6594
& 84.9896 & 0.2951 & 4.7530 & 24.9004 & 0.6926 \\

& BM3D~\cite{dabov2007bm3d}
& 30.7647 & 0.1449 & 3.2296 & 34.3446 & 0.9292
& 17.1473 & 0.0950 & 3.8459 & \textbf{34.9628} & \textbf{0.9564} \\

& N2C~\cite{WangDenoise2023}
& \underline{15.2816} & \underline{0.0710} & \underline{2.3475} & \underline{35.0344} & \underline{0.9323}
& \underline{9.3033} & \underline{0.0480} & 3.1467 & \underline{34.7184} & \underline{0.9515} \\

& N2N~\cite{lehtinen2018noise2noise}
& 15.5294 & 0.0760 & 2.4213 & \textbf{35.0567} & \textbf{0.9325}
& 9.8992 & 0.0508 & 3.1958 & 34.7343 & 0.9518 \\

& DD~\cite{heckel2018deep}
& 63.4353 & 0.3352 & 3.5795 & 27.4176 & 0.7881
& 40.2202 & 0.2862 & 4.6408 & 27.4888 & 0.7985 \\

& DeCompress~\cite{zafari2025decompress}
& 36.0712 & 0.1516 & 2.5735 & 29.4717 & 0.8494
& 21.5690 & 0.1399 & \underline{3.0150} & 29.5929 & 0.8704 \\

& \textbf{DiffDeCompress   }
& 15.4075 & 0.0722 & \textbf{2.2528} & 29.2386 & 0.8441
& 10.5465 & 0.0528 & \textbf{2.8880} & 28.9554 & 0.8780 \\

& \textbf{CGanDeCompress   }
& \textbf{12.9149} & \textbf{0.0495} & 2.6206 & 32.6086 & 0.9115
& \textbf{8.4437} & \textbf{0.0359} & 3.2627 & 31.5862 & 0.9315 \\

\hline\hline

% ========================= Sigma = 25 =========================
\multirow{8}{*}{$\sigma=25$}

& Noisy
& 157.7336 & 0.5638 & 5.7074 & 20.4386 & 0.4701
& 129.4306 & 0.4980 & 5.9286 & 20.6694 & 0.5225 \\

& BM3D
& 51.7303 & 0.2214 & 3.8550 & 31.1730 & 0.8771
& 32.4618 & 0.1601 & 4.3132 & 31.1874 & 0.9187 \\

& N2C
& \underline{23.3011} & 0.1106 & \underline{2.5141} & \textbf{32.6731} & \textbf{0.8948}
& \underline{16.7147} & \underline{0.0799} & 3.3093 & \textbf{32.2684} & \textbf{0.9217} \\

& N2N
& 27.6703 & 0.1223 & 2.6249 & \underline{32.6552} & \underline{0.8939}
& 20.0348 & 0.0882 & 3.4357 & \underline{32.2513} & \underline{0.9213} \\

& DD
& 83.7497 & 0.4210 & 3.6106 & 26.4950 & 0.7555
& 48.0839 & 0.3365 & 4.5886 & 26.6470 & 0.7732 \\

& DeCompress
& 65.0016 & 0.3651 & 3.1791 & 26.8033 & 0.7647
& 36.3198 & 0.2397 & \underline{3.2863} & 27.1986 & 0.8095 \\

& \textbf{DiffDeCompress   }
& 24.0302 & \underline{0.1076} & \textbf{2.2421} & 27.8702 & 0.8039
& 16.8973 & 0.0810 & \textbf{2.9319} & 27.5729 & 0.8428 \\

& \textbf{CGanDeCompress   }
& \textbf{20.8924} & \textbf{0.0742} & 2.7188 & 30.9462 & 0.8712
& \textbf{14.6683} & \textbf{0.0567} & 3.3520 & 29.9438 & 0.8991 \\

\hline\hline

% ========================= Sigma = 50 =========================
\multirow{8}{*}{$\sigma=50$}

& Noisy
& 280.5622 & 0.9680 & 8.2075 & 14.8716 & 0.2509
& 200.2227 & 0.8977 & 8.3045 & 15.1962 & 0.3047 \\

& BM3D
& 82.7551 & 0.3307 & 5.1690 & 28.8828 & 0.7935
& 62.0448 & 0.2659 & 5.2739 & 28.5877 & 0.8432 \\

& N2C
& 40.3229 & 0.1848 & 2.8118 & \textbf{29.6663} & \textbf{0.8228}
& 33.6093 & 0.1445 & 3.5777 & \textbf{29.0610} & \textbf{0.8584} \\

& N2N
& 42.4373 & 0.1772 & \underline{2.5716} & \underline{29.5478} & \underline{0.8201}
& \underline{33.4033} & \underline{0.1431} & \underline{3.4075} & \underline{28.9330} & \underline{0.8563} \\

& DD
& 137.4983 & 0.6104 & 3.8901 & 23.7966 & 0.6506
& 81.3767 & 0.4980 & 4.5994 & 23.8478 & 0.6858 \\

& DeCompress
& 130.1298 & 0.5541 & 4.2835 & 23.9714 & 0.6544
& 79.8575 & 0.4558 & 4.6101 & 23.5542 & 0.6839 \\

& \textbf{DiffDeCompress   }
& \textbf{34.9944} & \textbf{0.1076} & \textbf{2.2421} & 27.8702 & 0.8039
& 33.9302 & 0.1478 & \textbf{3.3153} & 25.5450 & 0.7783 \\

& \textbf{CGanDeCompress   }
& \underline{36.6453} & \underline{0.1299} & 2.9745 & 28.5292 & 0.7970
& \textbf{29.5434} & \textbf{0.1069} & 3.6599 & 27.5023 & 0.8319 \\

\hline
\end{tabular}

\caption{
\small Performance on KODAK and DIV2K with Gaussian noise $\mathcal{N}(0,\sigma^2)$ for $\sigma \in \{15,25,50\}$.
All models are trained on OpenImages.
Best results are shown in \textbf{bold}, and second-best results are \underline{underlined}.
}
\label{tab:kodak_div2k_sigma15_25_50}
\end{table*}

%===========================================================================================
\section{Experimental Results}\label{sec:experiments}
We evaluate the denoising performance of the proposed generative compression-based methods on both synthetic and real-world noise, using natural and microscopy image datasets.
Comparisons are conducted against representative classical, learning-based, and compression-based denoisers.
Baselines include the traditional BM3D~\cite{dabov2007bm3d}, learning-based methods Noise2Clean (N2C)~\cite{WangDenoise2023}, Noise2Noise (N2N)~\cite{lehtinen2018noise2noise}, and Deep Decoder (DD)~\cite{heckel2018deep}, as well as the recent compression-based approach DeCompress~\cite{zafari2025decompress}.

\textbf{Metric Evaluation.}
We report PSNR and SSIM~\cite{wang2004image} to evaluate distortion fidelity.
Perceptual quality is assessed using LPIPS, DISTS, FID, and PI, following standard practices in perceptual image quality evaluation~\cite{ding2020image}.
LPIPS is computed using an AlexNet-based feature extractor and correlates well with human perceptual judgments, while DISTS measures perceptual similarity by jointly modeling structural and texture information.
FID captures distributional discrepancies in deep feature space, and PI aggregates no-reference quality metrics, with lower values indicating better perceptual realism.

\subsection{Natural images with synthetic Gaussian noise}
\textbf{Dataset Description.}
All models are trained on the OpenImages dataset~\cite{OpenImages}.
For evaluation, we consider two standard natural image benchmarks: Kodak, which contains 24 high-quality images at resolution $768\times512$, and DIV2K, for which we use the validation set of 100 high-resolution images ~\cite{Agustsson_2017_CVPR_Workshops}.
For DIV2K, each image is resized such that its shorter side is 768 pixels, followed by center cropping to $768\times768$.

\textbf{Model Training.} 
Conditional WGAN-based compression model (\emph{CGanDeCompress}) are trained for $8$ epochs with batch size $8$ using the Adam optimizer. We use an initial learning rate of $1\times10^{-4}$ with weight decay $1\times10^{-6}$, and apply a step learning-rate schedule that decays the rate by a factor of $0.1$ after $5\times10^5$ steps. For comparison, the conditional diffusion-based compression model (\emph{DiffDeCompress}) employs $N_{\text{train}}=8{,}193$ diffusion steps with a cosine noise schedule, batch size $4$, and Adam optimization. Its learning rate is initialized at $5\times10^{-5}$, reduced by $20\%$ every $100{,}000$ steps, and clipped to $2\times10^{-5}$. All models are trained using the PyTorch framework on two NVIDIA RTX~3090 GPUs under Ubuntu.

Table~\ref{tab:kodak_div2k_sigma15_25_50} reports denoising results on KODAK and DIV2K under additive Gaussian noise with $\sigma\in\{15,25,50\}$. 
Across all noise levels, the proposed generative compression-based denoisers, CGanDeCompress and DiffDeCompress, consistently outperform classical and learning-based baselines in terms of perceptual quality metrics (FID, LPIPS, and PI), while exhibiting the expected trade-off with pixel-fidelity measures (PSNR and SSIM), which are typically favored by supervised denoisers such as N2C and N2N.

For light noise ($\sigma=15$), CGanDeCompress achieves the best perceptual similarity on both datasets, attaining the lowest FID and LPIPS on KODAK (FID $12.91$, LPIPS $0.0495$) and DIV2K (FID $8.44$, LPIPS $0.0359$). 
In contrast, N2C and N2N yield the highest PSNR/SSIM (e.g., $\approx35$ dB on KODAK), but with substantially higher FID and LPIPS, indicating less perceptually realistic reconstructions. DiffDeCompress achieves the lowest PI on both datasets (PI $2.25$ on KODAK and $2.89$ on DIV2K), highlighting its strength in perceptual realism while maintaining the competitive distortion performance.

As the noise level increases ($\sigma=25,50$), similar trends persist.
At $\sigma=25$, CGanDeCompress remains the strongest in FID and LPIPS on both datasets (e.g., DIV2K FID $14.67$, LPIPS $0.0567$), while DiffDeCompress consistently achieves the lowest PI (KODAK PI $2.24$, DIV2K PI $2.93$).
At the highest noise level ($\sigma=50$), DiffDeCompress provides the best perceptual quality on KODAK across all perceptual metrics, whereas CGanDeCompress achieves the best FID and LPIPS on DIV2K (FID $29.54$, LPIPS $0.1069$), with DiffDeCompress still yielding the lowest PI.

Qualitative comparisons in Fig.~\ref{fig:methods_comparison_kodak_25} and Fig.~\ref{fig:methods_comparison_div2k_50} illustrate representative denoising artifacts at noise level $\sigma=25$ on KODAK and $\sigma=50$ on DIV2K, respectively.
Compared to the proposed approaches, existing methods often yield over-smoothed reconstructions and fail to recover high-frequency textures. 
Such artifacts are typically linked to MSE/PSNR-based optimization, which biases estimates toward average image statistics and suppresses fine structures~\cite{blau2019rethinking}. 
In contrast, our method preserves subtle details and produces more perceptually faithful restorations. 

Overall, these results demonstrate that our frameworks enable a favorable perceptual denoising regime, with CGanDeCompress excelling in distributional and perceptual similarity by FID, and DiffDeCompress excelling in perceptual realism as quantified by PI.

% \begin{table}[t]
% \footnotesize
% \centering
% \setlength{\tabcolsep}{3.0pt}
% \renewcommand{\arraystretch}{1.05}

% \begin{tabular}{c|l|cc}
% \hline
% \textbf{Dataset}
% & \textbf{Method}
% & \textbf{LPIPS / DISTS}$\downarrow$
% & \textbf{PSNR / SSIM}$\uparrow$ \\
% \hline\hline

% \multirow{2}{*}{\textbf{FMD}}
% & Noisy
% & 0.439 / 0.344
% & 28.45 / 0.582 \\

% & \textbf{CGanDeCompress}
% & \textbf{0.050 / 0.140}
% & \textbf{35.14 / 0.916} \\

% \hline

% \multirow{2}{*}{\textbf{SIDD}}
% & Noisy
% & 0.646 / 0.318
% & 26.35 / 0.506 \\

% & \textbf{CGanDeCompress}
% & \textbf{0.104 / 0.121}
% & \textbf{32.24 / 0.889} \\

% \hline
% \end{tabular}

% \caption{
% \small Performance on real fluorescence microscopy (FMD) and real camera images (SIDD).
% }
% \label{tab:fmd_sidd_real_denoising}
% \end{table}

\begin{table*}[t]
\footnotesize
\centering
\setlength{\tabcolsep}{3.2pt}
\renewcommand{\arraystretch}{1.05}

\begin{tabular}{c|l|cccc}
\hline
\textbf{Dataset}
& \textbf{Method}
& \textbf{LPIPS}$\downarrow$
& \textbf{DISTS}$\downarrow$
& \textbf{PSNR}$\uparrow$
& \textbf{SSIM}$\uparrow$ \\
\hline\hline

\multirow{2}{*}{\textbf{FMD}}
& Noisy
& 0.439
& 0.344
& 28.45
& 0.582 \\

& \textbf{CGanDeCompress}
& \textbf{0.035}
& \textbf{0.140}
& \textbf{36.71}
& \textbf{0.934} \\

\hline

\multirow{2}{*}{\textbf{SIDD}}
& Noisy
& 0.646
& 0.318
& 26.35
& 0.506 \\

& \textbf{CGanDeCompress}
& \textbf{0.093}
& \textbf{0.121}
& \textbf{32.70}
& \textbf{0.907} \\

\hline
\end{tabular}

\caption{
\small Performance on real fluorescence microscopy (FMD) and real camera images (SIDD).
}
\label{tab:fmd_sidd_real_denoising}
\vspace{-0.4cm}
\end{table*}

\subsection{Fluorescence microscopy and real camera images}
\textbf{Training Datasets~\cite{WangDenoise2023}.}
We evaluate denoising performance under real noise and domain shift using fluorescence microscopy and real camera datasets.
For microscopy, we use the Fluorescence Microscopy Dataset (FMD), which contains approximately $12{,}000$ fluorescence microscopy images of different samples obtained with different microscopes.
The pseudo ground-truth is collected by the image average strategy. 
For real camera noise, we consider the Smartphone Image Denoising Dataset (SIDD), consisting of roughly $30{,}000$ noisy images across $10$ scenes under different lighting conditions, with clean references generated via multi-frame averaging.

Table~\ref{tab:fmd_sidd_real_denoising} reports denoising performance under real noise and severe domain shift on FMD and SIDD.
Compared to the noisy inputs, CGanDeCompress achieves consistent and substantial improvements across all reported metrics.
On FMD, LPIPS is reduced from $0.439$ to $0.035$ while PSNR increases from $28.45$\,dB to $36.71$\,dB, accompanied by a significant SSIM improvement from $0.582$ to $0.934$, indicating effective noise suppression while preserving fine structural details.

On SIDD, CGanDeCompress further improves perceptual and structural quality, reducing LPIPS from $0.646$ to $0.093$ and increasing PSNR from $26.35$\,dB to $32.70$\,dB.
Structural fidelity is substantially enhanced, with SSIM rising from $0.506$ to $0.907$, while perceptual distortion measured by DISTS decreases from $0.318$ to $0.121$.
These results demonstrate the robustness of the proposed conditional WGAN-based generative compression framework to complex real-world noise and challenging domain shifts.
Overall, CGanDeCompress generalizes effectively across both microscopy and camera imaging scenarios, producing denoised outputs that are perceptually faithful and structurally consistent under realistic noise conditions.
\vspace{-0.35cm}

%===========================================================================================
\section{Conclusions}
\label{sec:disc}
\vspace{-0.3cm}

This paper presented a generative compression framework for perception-based image denoising, where restoration is achieved by reconstructing from entropy-coded latent representations. 
Two complementary instantiations were introduced: CGanDeCompress, which controls the RDO trade-off via conditional adversarial training, and DiffDeCompress, which performs diffusion-based reconstruction guided by compressed latents to enhance perceptual realism. 
Non-asymptotic performance guarantees were also established for compression-based ML denoising under additive Gaussian noise, including bounds on reconstruction error and decoding error probability. 
Experiments on synthetic and real-noise benchmarks demonstrate consistent perceptual improvements while maintaining competitive distortion performance, highlighting the robustness of generative compression for perceptual denoising.

%===========================================================================================
\clearpage
\newpage
\bibliographystyle{IEEEbib}
\bibliography{main}

%===========================================================================================
\clearpage
\appendix

\section{Proof of Theoretical Result}
\label{app:additional_proofs}

%===========================================================================================
\subsection{Proof of Theorem \ref{thm:MSE_Gaussian_noise}}
\label{app:proof_MSE_Gaussian_noise}

\textbf{Part I.}  
The proof approach is based on the technique introduced in~\cite{zafari2025zeroshot}. Recall that the noisy observation is given by $\yv = \xv + \nv$, where $\nv \sim \iid \mathcal{N}(\mathbf{0}, \sigma^2 \mathbf{I}_n)$. Define the compression-based estimator and the oracle reconstruction as
\[
   \xvh = \arg \min_{\cv_m \in \Cc} \|\cv_m - \yv\|^2, 
   \quad  
   \xvt = \arg \min_{\cv_m \in \Cc} \|\cv_m - \xv\|^2 .
\]

Since both $\xvh$ and $\xvt$ belong to the codebook $\Cc$, the optimality of $\xvh$ implies
\begin{align*}
    &\|\hat{\xv} - \yv\|^2 \leq \|\Tilde{\xv} - \yv\|^2 \\
    &\|(\hat{\xv} - \xv) - \nv\|^2 \leq \|(\Tilde{\xv} - \xv) - \nv\|^2 \\
    &\|\hat{\xv} - \xv\|^2 
    \leq \|\Tilde{\xv} - \xv\|^2 
    - 2 \angles{\nv,\Tilde{\xv} - \xv} 
    + 2 \angles{\nv,\hat{\xv} - \xv} \\
    &\|\hat{\xv} - \xv\|^2 
    \leq \|\Tilde{\xv} - \xv\|^2 
    + 2 \abs*{\angles{\nv,\Tilde{\xv}-\xv}} 
    + 2 \abs*{\angles{\nv,\hat{\xv}-\xv}} .
\end{align*}

Let $\ev = \hat{\xv} - \xv$ denote the estimation error induced by compression from the noisy observation $\yv$, and let $\dv = \Tilde{\xv} - \xv$ denote the distortion incurred by compressing the clean signal $\xv$. Then,
\begin{align}
    \norm{\ev}^2 
    &\le \norm{\dv}^2 
    + 2\abs{\angles{\nv, \ev}} 
    + 2\abs{\angles{\nv, \dv}} \nonumber\\
    &= \norm{\dv}^2 
    + 2\norm{\ev}\abs*{\angles{\nv, \frac{\ev}{\norm{\ev}}}} 
    + 2\norm{\dv}\abs*{\angles{\nv, \frac{\dv}{\norm{\dv}}}} .
    \label{eq:error-bound-gaussian}
\end{align}

For any reconstruction $\cv_m \in \Cc$, define the corresponding error vector $\ev^{(\cv_m)} = \cv_m - \xv$. For $t_1,t_2 > 0$, introduce the events
\begin{align*}
    \Ec_1 
    &= \braces*{ 
    \abs*{\sum^n_{i=1} z_i \frac{e_i^{(\cv_m)}}{\norm{\ev^{(\cv_m)}}}} \leq t_1 
    :\; \cv_m \in\Cc }, \\
    \Ec_2 
    &= \braces*{ 
    \abs*{\sum^n_{i=1} z_i \frac{e_i^{(\xvt)}}{\norm{\ev^{(\xvt)}}}} \leq t_2 } .
\end{align*}

Conditioned on the event $\Ec_1 \cap \Ec_2$, inequality~\eqref{eq:error-bound-gaussian} yields
\begin{align}
    \norm{\ev}^2 
    \leq \norm{\dv}^2 
    + 2t_1\norm{\ev} 
    + 2t_2\norm{\dv}.
    \label{eq:error-bound-whp}
\end{align}

Rearranging terms, we obtain
\begin{align*}
    \norm{\ev}^2 - 2t_1\norm{\ev} + t_1^2
    &\le \norm{\dv}^2 + 2t_2\norm{\dv} + t_2^2 + (t_1^2 - t_2^2), \\
    \abs*{\norm{\ev} - t_1}
    &\le \sqrt{(\norm{\dv} + t_2)^2 + (t_1^2 - t_2^2)} .
\end{align*}

Consequently,
\begin{align}
    \|\ev\| 
    \le \|\dv\| + t_1 + t_2 + \sqrt{t_1^2 - t_2^2},
\end{align}
where the final step follows from the inequality $\sqrt{a+b} \le \sqrt{a} + \sqrt{b}$ for all $a,b>0$.

To complete the proof, it remains to bound $\P{(\Ec_1 \cap \Ec_2)^c}$ and choose $t_1,t_2$ such that $t_1^2 - t_2^2 \ge 0$.

Recall that
\[
\| \dv \|^2 
= \| \tilde{\xv} - \xv \|^2 
= D(P) 
= D^* + \bigl[(P^* - P)_+\bigr]^2 ,
\]
which implies
\begin{align}
    \|\ev\| 
    \le \sqrt{D(P)} 
    + t_1 + t_2 + \sqrt{t_1^2 - t_2^2}.
    \label{eq:bound-thm1-final}
\end{align}

For each $\cv_m \in \Cc$, the normalized vector $\ev^{(\cv_m)}/\norm{\ev^{(\cv_m)}}$ is a unit vector in $\R^n$. Hence,
\[
\sum^n_{i=1} z_i \frac{e_i^{(\cv_m)}}{\|\ev^{(\cv_m)}\|} 
\sim \mathcal{N}(0, \sigma^2),
\]
and therefore
\begin{align*}
    \P{\abs*{\sum^n_{i=1} z_i \frac{e_i^{(\cv_m)}}{\norm{\ev^{(\cv_m)}}}} \ge t}
    \le 2 \exp\parens*{-\frac{t^2}{2\sigma^2}} .
\end{align*}

Applying the union bound and using $|\Cc| \le 2^R$, we obtain
\begin{align*}
    \P{\Ec_1^c} 
    \le 2^{R+1} \exp\parens*{-\frac{t_1^2}{2\sigma^2}}, 
    \quad
    \P{\Ec_2^c} 
    \le 2 \exp\parens*{-\frac{t_2^2}{2\sigma^2}} .
\end{align*}

For any $\eta \in (0,1)$, choose
\[
t_1 = \sigma \sqrt{2 \ln 2 \, R (1+\eta)}, 
\quad
t_2 = \sigma \sqrt{2 \ln 2 \, R \eta}.
\]
Then,
\begin{align*}
    \P{\Ec_1^c \cup \Ec_2^c}
    \le 2^{-\eta R + 2}.
\end{align*}

Substituting these values of $t_1$ and $t_2$ into \eqref{eq:bound-thm1-final} yields
\begin{align*}
    \|\xv - \xvh\|_2 
    \le \sqrt{D(P)} 
    + 2\sigma \sqrt{(2 \ln 2) R} (1 + 2\sqrt{\eta}).
\end{align*}

Moreover, since $\xvt$ minimizes distortion over $\Cc$, we have
\begin{align*}
    \| \hat{\xv} - \xv \|^2 
    \ge \|\Tilde{\xv} - \xv\|^2 
    = D(P).
\end{align*}

Combining the upper and lower bounds completes the proof:
\begin{align*}
\sqrt{D(P)} \leq \|\xv-\xvh\| \le \sqrt{D(P)}  + 2\sigma \sqrt{(2 \ln2) R}(1+2\sqrt{\eta}),
\end{align*}

\textbf{Part II.}  
The argument follows a standard approach adapted from~\cite{Goldsmith2005}. We define a collection of decision regions $(Z_1,\ldots,Z_{2^R})$, each corresponding to a subset of the signal space $\R_+^n$, as
\begin{align*}
    Z_m = (\yv: P(\cv_m \text{ is optimal} | \yv) > P(\cv_k \text{ is optimal} | \yv)), \forall k \neq m.
\end{align*}

By construction, these regions are mutually exclusive. We further assume that there exists no $\yv \in \R_+^n$ such that $P(\cv_m \text{ is optimal} | \yv) = P(\cv_k \text{ is optimal} | \yv)$ for any $m \neq k$, ensuring that the codebook space is fully partitioned by the decision regions. Given a noisy observation $\yv \in Z_m$, the optimal denoiser therefore outputs the estimate $\hat{\xv} = \cv_m$. Under the compression-based ML denoiser, the decision regions can equivalently be written as
\begin{align*}
    Z_m = (\yv: \| \yv - \cv_m \| < \| \yv - \cv_k \|), \\ 
    \forall k \in \{ 1, ..., 2^R \}, \hspace{0.01cm} k \neq m, \hspace{0.01cm} 
    m \in \{ 1, ..., 2^R \}.
\end{align*}
That is, the decoded signal is chosen as the codeword closest to the noisy observation in Euclidean distance.

We now characterize the error probability of the ML denoiser.  
Assuming an equally likely codebook, i.e., $P(\cv_m \text{ is optimal}) = 1/2^R$, the probability of error is given by
\begin{align}
 P_e 
&= \sum_{m=1}^{2^R} P(\yv \notin Z_m | \cv_m \text{ is optimal}) P(\cv_m \text{ is optimal}) \notag \\
&= \frac{1}{2^R} \sum_{m=1}^{2^R} P(\yv \notin Z_m | \cv_m \text{ is optimal}) \notag \\
&= 1 - \frac{1}{2^R} \sum_{m=1}^{2^R} P(\yv \in Z_m | \cv_m \text{ is optimal}) \notag \\
&= 1 - \frac{1}{2^R} \sum_{m=1}^{2^R} \int_{Z_m} 
    P(\yv | \cv_m \text{ is optimal}) \, d\yv \notag \\
&= 1 - \frac{1}{2^R} \sum_{m=1}^{2^R} \int_{Z_m} 
    P(\yv = \xv + \nv | \cv_m \text{ is optimal}, \xv) \, d\nv \notag \\
&= 1 - \frac{1}{2^R} \sum_{m=1}^{2^R} \int_{Z_m} 
    P(\yv = \cv_m - \dv + \nv | \cv_m \text{ is optimal}, \xv) \, d\nv \notag \\
&= 1 - \frac{1}{2^R} \sum_{m=1}^{2^R} 
    \int_{Z_m - \cv_m + \dv} p(\nv) \, d\nv. 
\label{eq:Pe_expression}
\end{align}
Here, $\cv_m = g(f(\xv)) = \arg \min_{\cv_k \in\Cc} \|\cv_k - \xv\|^2$, and $\dv = \cv_m - \xv$.

The integrals in~\eqref{eq:Pe_expression} are taken over $n$-dimensional subsets $Z_m \subset \R_+^n$. Let $A_{mv}$ denote the event that $\|\yv - \cv_v\| < \|\yv - \cv_m\|$ conditioned on $\cv_m$ being optimal. When $A_{mv}$ occurs, the decoder selects $\cv_v$ instead of the correct codeword $\cv_m$, resulting in a decoding error. Correct decoding requires that $\|\yv - \cv_m\| < \|\yv - \cv_v\|$ for all $v \neq m$. Applying the union bound, we obtain
\begin{align}
P_e(\cv_m \text{ is optimal}) 
&= P\!\left( \bigcup_{v = 1, v \neq m}^{2^R} A_{mv} \right) 
\le \sum_{v = 1, v \neq m}^{2^R} P(A_{mv}),
\label{eq:Pe_union_bound}
\end{align}

where
\begin{align*}
P(A_{mv}) 
&= P(\|\cv_v - \yv\| < \|\cv_m - \yv\| | \cv_m \text{ is optimal}) \notag \\
&= P(\|\cv_v - (\cv_m - \dv + \nv)\| 
      < \|\cv_m - (\cv_m - \dv + \nv)\|) \notag \\
&= P(\|\nv - \dv + \cv_m - \cv_v\| 
      < \|\nv - \dv\|) \notag \\
&= P(\|\nv - \dv + \cv_m - \cv_v\|^2 
      < \|\nv - \dv\|^2).
\label{eq:pAik}
\end{align*}

Define $\wv = \nv - \dv$, where $\nv \sim \mathcal{N}(\mathbf{0}, \sigma^2 \mathbf{I}_n)$. Since $\dv$ is deterministic given $\xv$ and $\cv_m$, it follows that $\wv \sim \mathcal{N}(-\mathbf{d}, \sigma^2 \mathbf{I}_n)$. The inequality above can be rewritten as
\begin{align*}
\|\wv + \cv_m - \cv_v\|^2 &< \|\wv\|^2 \\
\|\wv\|^2 - 2 \angles{\wv, \cv_v - \cv_m} 
    + \|\cv_m - \cv_v\|^2 &< \|\wv\|^2 \\
\angles{\wv, \cv_v - \cv_m} 
&> \frac{\|\cv_m - \cv_v\|^2}{2} \\
\angles{\wv, \frac{\cv_v - \cv_m}{\|\cv_m - \cv_v\|}} 
&> \frac{\|\cv_m - \cv_v\|}{2}.
\end{align*}
Hence,
\begin{align*}
P(A_{mv}) 
&= P\!\left( 
\angles{\wv, \frac{\cv_v - \cv_m}{\|\cv_m - \cv_v\|}} 
> \frac{\|\cv_m - \cv_v\|}{2} 
\right).
\end{align*}

This probability depends only on the projection of $\wv$ onto the line connecting the origin and $\cv_m - \cv_v$. Let $w$ denote this one-dimensional projection, which is Gaussian with mean 
$\angles{-\dv, \frac{\cv_v - \cv_m}{\|\cv_m - \cv_v\|}}$ and variance $\sigma^2$. The event $A_{mv}$ occurs when $w > \frac{\|\cv_m - \cv_v\|}{2}$, yielding
\begin{align*}
P(A_{mv}) 
&= P\left( w > \frac{\|\cv_m - \cv_v\|}{2} \right) \\
&= \int_{\|\cv_m - \cv_v\|/2}^{\infty} 
   \frac{1}{\sqrt{\pi 2\sigma^2}} 
   \exp\!\left[
   -\frac{\left(
   h + \angles{\dv, \frac{\cv_v - \cv_m}{\|\cv_m - \cv_v\|}}
   \right)^2}{2\sigma^2}
   \right] dh \\
&= Q\!\left(
\frac{\frac{\|\cv_m - \cv_v\|}{2} 
+ \angles{\dv, \frac{\cv_v - \cv_m}{\|\cv_m - \cv_v\|}}}
{\sigma}
\right),
\end{align*}
where $Q(h) = \int_{h}^{\infty} \frac{1}{\sqrt{2\pi}} e^{-x^2/2} \, dx$.

Substituting into~\eqref{eq:Pe_union_bound} gives
\begin{align*}
P_e(\cv_m \text{ is optimal}) 
&\le \sum_{v \neq m}^{2^R} 
Q\!\left(
\frac{\frac{\|\cv_m - \cv_v\|}{2} 
+ \angles{\dv, \frac{\cv_v - \cv_m}{\|\cv_m - \cv_v\|}}}
{\sigma}
\right).
\end{align*}

Finally, averaging over all codewords yields the union bound
\begin{align*}
P_e 
&= \sum_{m=1}^{2^R} P(\cv_m) P_e(\cv_m \text{ is optimal}) \\
&\le \frac{1}{2^R} \sum_{m=1}^{2^R} 
\sum_{\substack{v=1 \\ v \neq m}}^{2^R} 
Q\!\left(
\frac{\frac{\|\cv_m - \cv_v\|}{2} 
+ \angles{\dv, \frac{\cv_v - \cv_m}{\|\cv_m - \cv_v\|}}}
{\sigma}
\right).
\end{align*}  

\section{Additional experiment results}
\label{app:additional_experiments}

\begin{table*}[t]
\footnotesize
\centering
\setlength{\tabcolsep}{3.0pt}        % tighter columns
\renewcommand{\arraystretch}{1.05}   % tighter rows

\begin{tabular}{c|l|ccccc}
\hline
\multirow{2}{*}{\textbf{Noise}}
& \multirow{2}{*}{\textbf{Method}}
& \multicolumn{5}{c}{\textbf{COCO2017}} \\
\cline{3-7}
&
& \textbf{FID}$\downarrow$
& \textbf{LPIPS}$\downarrow$
& \textbf{PI}$\downarrow$
& \textbf{PSNR}$\uparrow$
& \textbf{SSIM}$\uparrow$ \\
\hline\hline

% ========================= Sigma = 15 =========================
\multirow{8}{*}{$\sigma=15$}

& Noisy
& 30.4841 & 0.2806 & 3.2707 & 27.2663 & 0.6972 \\

& BM3D~\cite{dabov2007bm3d}
& 32.6753 & 0.1454 & 3.8140 & 30.7513 & 0.8789 \\

& N2C~\cite{WangDenoise2023}
& 11.3022 & 0.0670 & 3.2584 & \underline{31.3145} & 0.9140 \\

& N2N~\cite{lehtinen2018noise2noise}
& 11.2186 & 0.0680 & 3.2844 & \textbf{31.3311} & \underline{0.9143} \\

& DD~\cite{heckel2018deep}
& 30.1372 & 0.2152 & 3.8365 & 28.2037 & 0.8274 \\

& DeCompress~\cite{zafari2025decompress}
& 27.1170 & 0.1152 & \textbf{2.5488} & 28.6527 & 0.8415 \\

& \textbf{DiffDeCompress}
& \underline{10.0719} & \underline{0.0646} & \underline{2.9227} & 28.2667 & 0.8533 \\

& \textbf{CGanDeCompress}
& \textbf{6.0210} & \textbf{0.0332} & 3.3421 & 31.2392 & \textbf{0.9161} \\

\hline\hline

% ========================= Sigma = 25 =========================
\multirow{8}{*}{$\sigma=25$}

& Noisy
& 45.0075 & 0.4890 & 6.1671 & 23.0797 & 0.4961 \\

& BM3D
& 40.7810 & 0.1961 & 4.3051 & 29.7169 & 0.8490 \\

& N2C
& 15.2267 & 0.0944 & 3.4406 & \underline{30.2771} & \textbf{0.8888} \\

& N2N
& 16.2335 & 0.1006 & 3.5540 & \textbf{30.2831} & \underline{0.8885} \\

& DD
& 51.9262 & 0.2780 & 3.7112 & 26.9307 & 0.7893 \\

& DeCompress
& 40.4946 & 0.2096 & \textbf{2.8647} & 26.6663 & 0.7825 \\

& \textbf{DiffDeCompress}
& \underline{14.3620} & \underline{0.0886} & \underline{3.0277} & 27.1689 & 0.8231 \\

& \textbf{CGanDeCompress}
& \textbf{9.1813} & \textbf{0.0539} & 3.4899 & 29.8377 & 0.8839 \\

\hline\hline

% ========================= Sigma = 50 =========================
\multirow{8}{*}{$\sigma=50$}

& Noisy
& 80.2041 & 0.8733 & 9.8083 & 17.7498 & 0.2897 \\

& BM3D
& 54.2648 & 0.2797 & 5.2882 & 27.7774 & 0.7916 \\

& N2C
& 22.6049 & 0.1521 & 3.7821 & \textbf{28.3397} & \textbf{0.8358} \\

& N2N
& 23.8651 & 0.1520 & \underline{3.6438} & \underline{28.2550} & \underline{0.8340} \\

& DD
& 82.7493 & 0.4418 & 3.5764 & 23.5001 & 0.6733 \\

& DeCompress
& 75.9995 & 0.4142 & 4.2321 & 23.3746 & 0.6726 \\

& \textbf{DiffDeCompress}
& \underline{22.5364} & \underline{0.1521} & \textbf{3.2925} & 25.5890 & 0.7694 \\

& \textbf{CGanDeCompress}
& \textbf{15.8919} & \textbf{0.1037} & 3.8198 & 27.5443 & 0.8219 \\

\hline
\end{tabular}

\caption{
\small Performance on COCO2017 with Gaussian noise $\mathcal{N}(0,\sigma^2)$ for $\sigma \in \{15,25,50\}$.
All models are trained on OpenImages.
Best results are shown in \textbf{bold}, and second-best results are \underline{underlined}.
}
\label{tab:coco2017_sigma15_25_50}
\end{table*}

Table~\ref{tab:coco2017_sigma15_25_50} reports denoising performance on COCO2017 under additive Gaussian noise with $\sigma\in\{15,25,50\}$.
Across all noise levels, the proposed generative compression-based denoisers, CGanDeCompress and DiffDeCompress, consistently achieve superior perceptual quality compared to classical, supervised, and recent compression-based baselines, as reflected by improvements in FID, LPIPS, and PI.
At the same time, distortion-driven methods such as N2C and N2N attain the highest PSNR/SSIM, illustrating the expected perceptual--distortion trade-off.

For light noise ($\sigma=15$), CGanDeCompress achieves the best perceptual similarity, obtaining the lowest FID and LPIPS (FID $6.02$, LPIPS $0.0332$) while also matching the top distortion performance (SSIM $0.9161$).
DiffDeCompress provides the second-best FID/LPIPS and improves perceptual realism compared to supervised baselines.
In contrast, N2C/N2N maximize PSNR ($\approx31.3$\,dB) but exhibit noticeably higher perceptual distances, indicating less realistic texture recovery.

At moderate noise ($\sigma=25$), the proposed methods maintain clear perceptual advantages.
CGanDeCompress achieves the best FID and LPIPS (FID $9.18$, LPIPS $0.0539$), while DiffDeCompress remains competitive and yields strong perceptual scores.
Although N2C/N2N preserve higher PSNR/SSIM, their perceptual metrics are consistently weaker, highlighting the benefit of compression-guided generative reconstruction.

Under severe noise ($\sigma=50$), perceptual gains become more pronounced.
CGanDeCompress achieves the best distributional quality and perceptual similarity (FID $15.89$, LPIPS $0.1037$), whereas DiffDeCompress attains the lowest PI (PI $3.29$), suggesting enhanced perceptual realism under heavy corruption.
Overall, these results demonstrate that generative compression enables robust perceptual denoising on challenging natural-image benchmarks, producing visually faithful reconstructions while maintaining competitive distortion performance.

\begin{figure*}[t]
    \centering
    \includegraphics[width=0.9\textwidth]{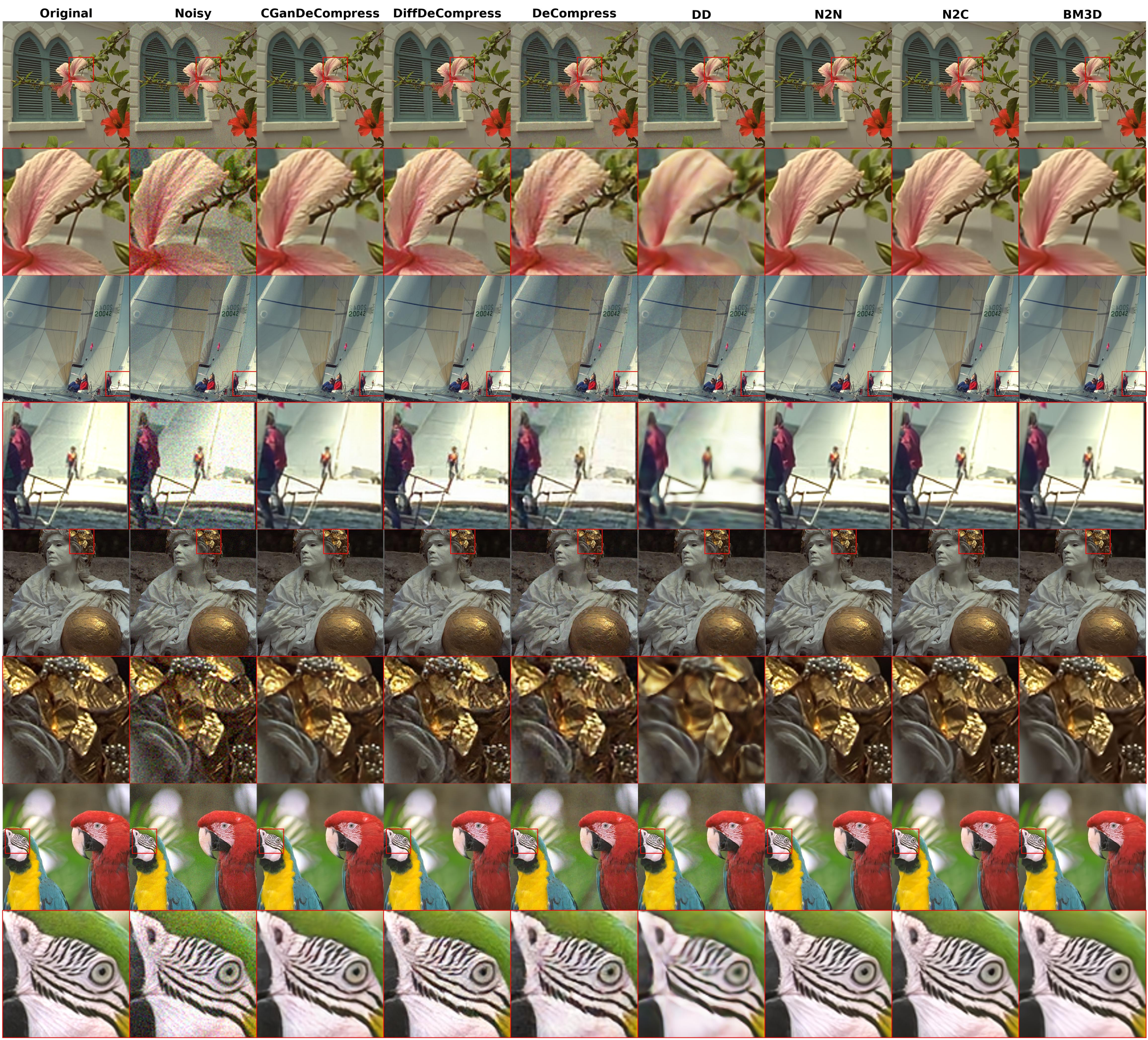}
    \caption{\small Visual comparison of denoising methods on randomly selected KODAK images at noise level $\sigma=15$.
    The first two columns present the original images and their noisy observations, followed by the denoised results generated by our methods and popular baselines.
    Zoomed-in regions are shown in the last two rows for detailed inspection.
    Best viewed on screen.}
    \label{fig:methods_comparison_kodak_15}
\end{figure*}

\begin{figure*}[t]
    \centering
    \includegraphics[width=0.9\textwidth]{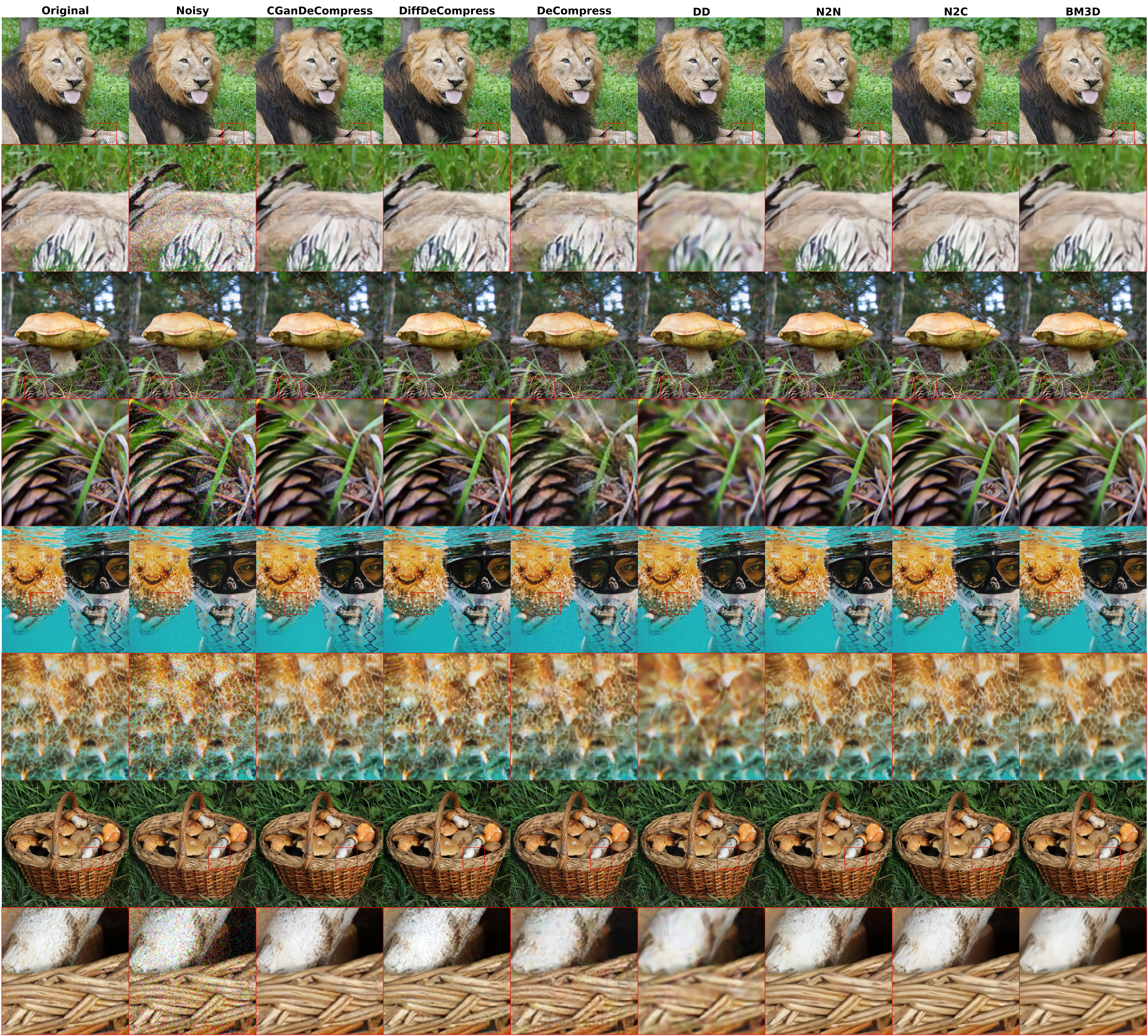}
    \caption{\small Visual comparison of denoising methods on randomly selected DIV2K images at noise level $\sigma=25$.
    The first two columns present the original images and their noisy observations, followed by the denoised results generated by our methods and popular baselines.
    Zoomed-in regions are shown in the last two rows for detailed inspection.
    Best viewed on screen.}
    \label{fig:methods_comparison_div2k_25}
\end{figure*}

\begin{figure*}[t]
    \centering
    \includegraphics[width=0.9\textwidth]{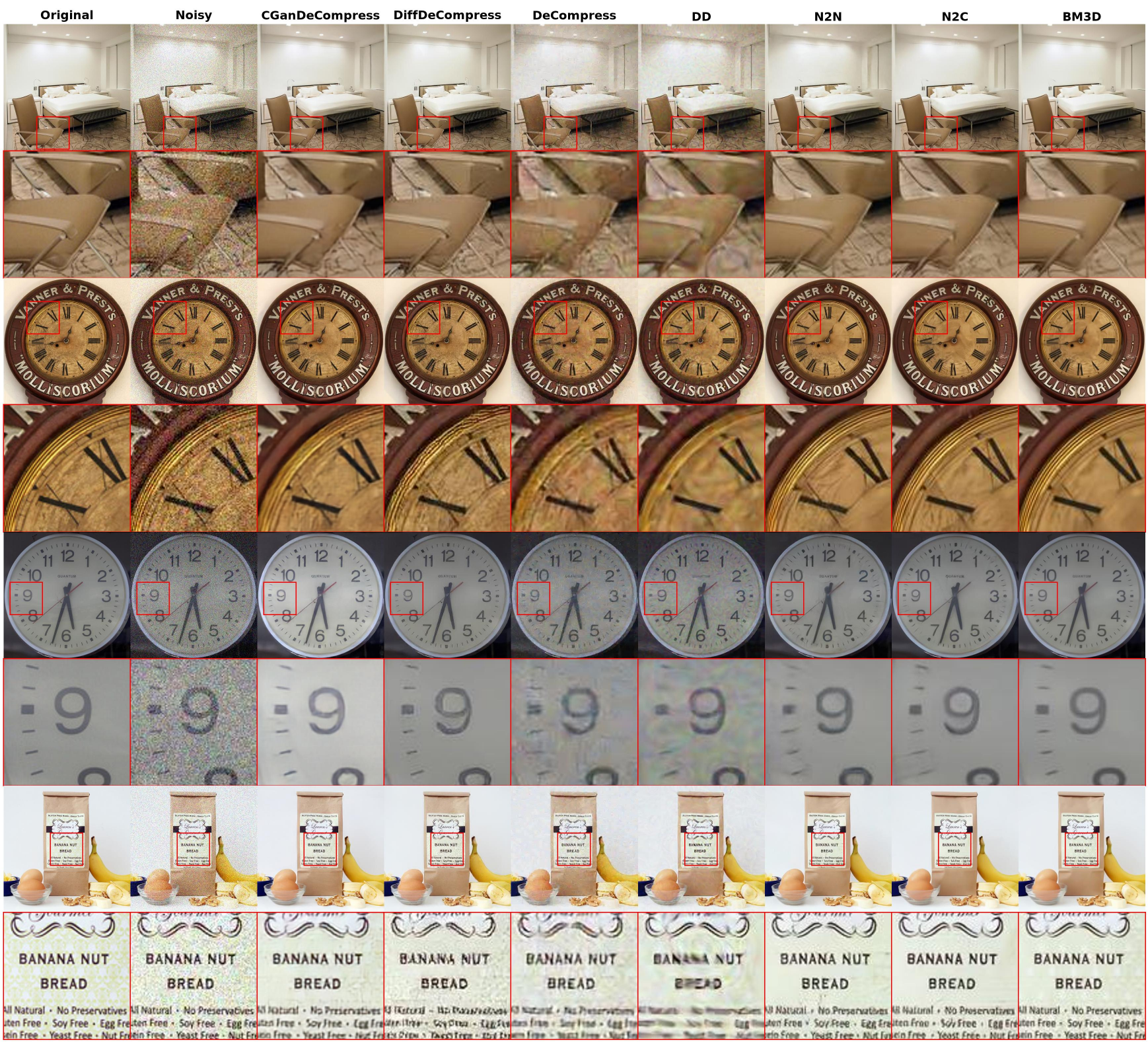}
    \caption{\small Visual comparison of denoising methods on randomly selected COCO2017 images at noise level $\sigma=25$.
    The first two columns present the original images and their noisy observations, followed by the denoised results generated by our methods and popular baselines.
    Zoomed-in regions are shown in the last two rows for detailed inspection.
    Best viewed on screen.}
    \label{fig:methods_comparison_coco2017_25}
\end{figure*}

\end{document}